\documentclass[runningheads]{llncs}
\usepackage[table]{xcolor}
\usepackage{tikz}
\usetikzlibrary{positioning, automata}
\usepackage[ruled, lined]{algorithm2e}
\usepackage{multirow}
\usepackage{mathtools}

\usepackage{amsthm}
\usepackage{pgfplots}
\usepackage{lineno}
\usepackage{tcolorbox}
\usepackage{subcaption}
\usepackage{bm}
\usepackage[utf8]{inputenc} 
\usepackage{hyperref}       
\usepackage{url}            
\usepackage{booktabs}       
\usepackage{amsfonts}       
\usepackage{nicefrac}       
\usepackage{microtype} 
\usepackage[numbers,sectionbib]{natbib}
\usetikzlibrary{decorations.pathreplacing}
\usepackage{graphicx}
\usepackage{standalone}

\usepackage{tabularx}
\usepackage{array}

\usepackage{wrapfig}
\usepackage{needspace}
\usepackage{capt-of}
\usepackage{tikz}
\usepackage{tikz-3dplot}
\usepackage{tcolorbox,varwidth}
\newtcbox{\tightbox}{colback=white,arc=0pt,boxsep=0pt,left=0pt,right=0pt,top=0pt,bottom=0pt,boxrule=0.5pt,varwidth upper}
\usepackage{amsmath}
\usepackage{pdflscape}
\usepackage{rotating}
\usepackage{appendix}
\usepackage{colortbl}
\usepackage{pifont}

\usepackage{wrapfig}

\usepackage{amsmath,amssymb}

\usepackage{textcomp,marvosym}

\usepackage{nameref}

\usepackage{xargs}                      

\usepackage{array}

\usepackage{listings} 

\usepackage{subcaption}

\usepackage[british]{babel}
\usepackage{hhline}
\usepackage{multirow}

\usepackage{algorithmic}
\usepackage{algorithm}
\usepackage{diagbox}

\usepackage{stmaryrd}

\usepackage{thmtools}
\usepackage{thm-restate}

\usepackage{adjustbox}

\usepackage{xr}
\usepackage{catchfilebetweentags}

\usepackage{tkz-orm}
\usetikzlibrary{arrows,shapes,automata,backgrounds,petri,decorations.markings,spy}
\usetikzlibrary{matrix,positioning,decorations.pathreplacing,calc,tikzmark}

\tikzset{
  ln/.style  = { draw, thick, fill=black, circle, inner sep=0.3mm},
  sqloc/.style  = { draw, thick, inner sep=1.0mm },
  loc/.style    = { draw, circle, thick, inner sep=1.0mm },
  notice/.style= { draw, rectangle callout, thick, rounded corners=5pt,fill=blue!10,callout relative pointer={#1} },
  treenode/.style = {align=center, inner sep=0pt, text centered,
    font=\sffamily},
  arn_n/.style = {treenode, rectangle, font=\sffamily\bfseries },
  mymat/.style = { 
    left delimiter={[}, right delimiter ={]},nodes={anchor=base east} },
  align at top/.style={baseline=(current bounding box.north)},
  stack/.style={rectangle split, rectangle split parts=#1,draw, anchor=center}
}

\usepackage{verbatim}
	
\def\cca#1{\cellcolor{blue!#10}\ifnum #1>5\color{white}\fi{#1}}

\usepackage{subcaption}
\usepackage{tikz,pgfplots}
\usetikzlibrary{tikzmark,decorations.pathreplacing,arrows,shapes,positioning,shadows,trees,shapes.gates.logic.US,arrows.meta,shapes,automata,petri,calc}
\usepackage{caption}
\usepackage{adjustbox}
\usepackage[colorinlistoftodos,prependcaption,textsize=tiny]{todonotes}
\newcommandx{\unsure}[2][1=]{\todo[linecolor=red,backgroundcolor=red!25,bordercolor=red,#1]{#2}}
\newcommandx{\change}[2][1=]{\todo[linecolor=blue,backgroundcolor=blue!25,bordercolor=blue,#1]{#2}}
\newcommandx{\info}[2][1=]{\todo[linecolor=OliveGreen,backgroundcolor=OliveGreen!25,bordercolor=OliveGreen,#1]{#2}}
\newcommandx{\improve}[2][1=]{\todo[linecolor=Plum,backgroundcolor=Plum!25,bordercolor=Plum,#1]{#2}}
\usepackage{verbatim}
\usepackage{scalefnt}
\usepackage{tikz-qtree}

\mathchardef\mhyphen="2D
\newcommand{\shorteq}{%
  \settowidth{\@tempdima}{-}
  \resizebox{\@tempdima}{\height}{=}%
}

\colorlet{fv}{gray!55}
\colorlet{ai}{gray!10}
\colorlet{ar}{gray!38}

\usepackage[edges]{forest}
\usepackage[T1]{fontenc}

\tikzset{%
  parent/.style={align=center,text width=3cm,rounded corners=3pt},
  child/.style={align=center,text width=3cm,rounded corners=3pt},
}

\def\addlegendimage{\csname pgfplots@addlegendimage\endcsname}

\newcommand*{\equal}{=}
 \usepackage{tikz}
 \usepackage{xparse}
\pgfdeclarelayer{myback}
\pgfsetlayers{myback,background,main}
\usetikzlibrary{tikzmark,decorations.pathreplacing,arrows,shapes,positioning,shadows,trees,shapes.gates.logic.US,arrows.meta,shapes,automata,petri,calc,matrix,backgrounds}
\tikzset{mycolor/.style = {line width=1bp,color=#1}}%
\tikzset{myfillcolor/.style = {draw,fill=#1}}%

\NewDocumentCommand{\highlight}{O{blue!40} m m}{%
\draw[mycolor=#1] (#2.north west)rectangle (#3.south east);
}

\NewDocumentCommand{\fhighlight}{O{blue!40} m m}{%
\draw[myfillcolor=#1] (#2.north west)rectangle (#3.south east);
}
 \usetikzlibrary{matrix,decorations.pathreplacing, calc, positioning}

\newcolumntype{+}{!{\vrule width 2pt}}

\newlength\savedwidth

\usepackage{array}
\newcolumntype{L}[1]{>{\raggedright\let\newline\\\arraybackslash\hspace{0pt}}m{#1}}
\newcolumntype{C}[1]{>{\centering\let\newline\\\arraybackslash\hspace{0pt}}m{#1}}
\newcolumntype{R}[1]{>{\raggedleft\let\newline\\\arraybackslash\hspace{0pt}}m{#1}}

\newcommand{\M}{\mathcal{M}}

\newcommand{\Le}{\mathcal{L}}
\newcommand{\I}{\mathcal{I}}

\newcommand{\ourtool}{\textsc{Ensense-0.2}}

\newtheorem{df}{Definition}

\mathchardef\mhyphen="2D

\newcommand{\tree}{\mathcal{T}}
\newcommand{\intnodes}{\mathit{IntNodes}}
\newcommand{\leaves}{\mathit{Leaves}}
\newcommand{\node}{n}
\newcommand{\pred}{\mathit{Pred}}
\newcommand{\true}{\mathit{True}}
\newcommand{\false}{\mathit{False}}
\newcommand{\const}{\mathit{Val}}
\newcommand{\sem}[1]{\llbracket #1\rrbracket}

\newcommand{\TEGLITCHmax}{\mathrm{TE\_GLITCH}}
\newcommand{\TEGLITCH}{\mathrm{TE\_GLITCH(\alpha)}}
\newcommand{\TEGLITCHi}{\mathrm{TE\_GLITCH(\alpha,i)}}

\renewcommand{\todo}[1]{{\color{orange}\textbf{TODO:}#1}}

\usepackage[normalem]{ulem}

\begin{document}
\title{Verifying Glitches in Decision Tree Ensembles}
%
%
\author{Satyankar Chandra\inst{1}\orcidID{0009-0003-3635-904X} \and\\
Harsh Chaurasia\inst{1}\orcidID{0009-0009-5862-6501} \and
Ashutosh Gupta\inst{1}\orcidID{0009-0003-7755-2006}\and 
Kaushik Mallik \inst{2}\orcidID{0000-0001-9864-7475}\and\\
Krishna Shankaranarayanan\inst{1}\orcidID{0000-0003-0925-398X}\and\\
Namrita Varshney \inst{1}(\Letter)\orcidID{0009-0006-7894-7896}}
%
%
\institute{Indian Institute Of Technology, Bombay, India \and
IMDEA Software Institute, Madrid, Spain \\
\email{namrita@iitb.ac.in }}
\maketitle              

\begin{abstract}
Decision tree ensembles are widely used for tabular prediction tasks in judiciary, healthcare, banking, etc.
We show that the piecewise-constant structure of these models can create narrow local oscillations, leading the ensemble score to sharply drop and then rise (a downward spike), or sharply rise and then drop (an upward spike), over a small change in a single feature.
We call such patterns \textit{glitches}. A glitch is not, by definition, an error. Rather, it is a formally checkable audit target requiring inspection.
Formally, glitches are witnessed by ordered input triples in which a small monotonic change in a single input feature causes the output to first change in one direction and then revert in the opposite direction.
Through a substantial experimental study, we show that glitches have widespread existence in even the state-of-the-art models performing various critical tasks.
Our findings also indicate that glitches enrich and extend existing notions like robustness and monotonicity, providing us a more refined tool to isolate anomalies for inspection.
On the theoretical side, we prove that deciding whether a decision tree ensemble
contains a glitch is NP-complete, even for bounded-depth trees. We then present
a mixed-integer linear programming encoding for the global and local glitch
verification problems, and implement our algorithm in a verifier for gradient-boosted decision trees. 
Experiments show our verifier can scale to realistic model sizes from the literature.

\end{abstract}

\section{Introduction}\label{sec:intro}

As AI-based systems are increasingly deployed in critical tasks \citep{berk2019MLcriminalJustice}\citep{mahajan2023ensembleForDisease}\citep{chouldechova2018AlgoCaseStudy}, the need for post-training verification of their behavior is now greater than ever. While a formal specification of trustworthiness remains missing, informally, the goal is to make sure that the outputs of these systems remain ``consistent'' over varying inputs.
One such metric for measuring consistency is
\textit{robustness}~\citep{leino2021globally,chen2021learning}, which requires that
``similar'' inputs must produce ``similar'' outputs. Another example is
\textit{monotonicity}~\citep{ben2009adding,sharma2020higher}, which asks whether the
output changes monotonically with respect to monotone changes in certain input
features. 
Neither of these is a foolproof measure of trustworthiness: model outputs do
change across decision boundaries, and non-monotonic behaviors are not always incorrect.

In this paper, we introduce a new inconsistency metric that strengthens both robustness and monotonicity.
Our observation is that many times, models exhibit output oscillations in which
a small monotonic change in a single input feature causes the output to first
jump in one direction and then jump back in the opposite direction.
We call such a pattern a \emph{glitch}; see Figure~\ref{fig:glitch versus robustness concept} for a visual illustration. 
Just like robustness and monotonicity violations, glitches do not automatically indicate a model error.
Rather, they are useful \emph{audit witnesses} as they identify localized decision pulses that may warrant further inspection by a model developer or domain expert. 
As an illustrative example, consider an AI model built to screen graduate student applications based on the grades.
The following situation would qualify as a glitch: two students with close
grades $8.6$ and $8.7$ are rejected, but a third student with the intermediate
grade $8.65$ is accepted, assuming all other feature values remain the same
across all three students. This is an oscillation in outcomes
(``reject''$\to$``accept''$\to$``reject'') under a small monotonic change in the
input feature.
Such a behavior in this example's context could indicate an anomaly; in other
contexts, this may not immediately indicate an anomaly, but regardless, we
submit that it
exposes a model behavior that requires a closer inspection. 
\begin{figure}[h]
\begin{center}
    \renewcommand{\Cross}{$\mathbin{\tikz [x=1.4ex,y=1.4ex,line width=.2ex, red] \draw (0,0) -- (1,1) (0,1) -- (1,0);}$}%
\begin{tikzpicture}[line cap=round,line join=round,>=Stealth,draw=black,scale=0.45]

\draw[thick,->] (0,0) -- (9,0) node[right]{\small input (monotonically increasing)};
\draw[thick,->] (0,0) -- (0,6) node[above]{\small output score};

\draw[dashed] (0,3) -- (9,3);

\draw[dashed] (0,3) -- (9,3)
  node[pos=1.15,above=3mm]{\footnotesize $\text{prediction}=A$}
  node[pos=1.15,below=3mm]{\footnotesize $\text{prediction}=B$};

\draw[<->] (9,4) -- (9,2);

\draw[very thick,blue] (1.5,3) circle (3mm);
\draw[very thick,blue] (7,3) circle (3mm);
\node[red]              at  (3.9,5.2)   {\Cross};
\node[red]              at  (3.6,1.2)   {\Cross};
\node[red]              at  (4.2,1.2)   {\Cross};
\draw[black]
  (0,5.0) -- (1.5,5.0)   
  -- (1.5,1.2)             
  -- (3.7,1.2)           
  -- (3.7,5.2)             
  -- (4.1,5.2)
  -- (4.1,1.2)          
  -- (5,1.2)
  -- (7,1.2)
  -- (7,5.0)             
  -- (10,5.0);            

\end{tikzpicture}
\end{center}
\caption{Glitch illustration: 
The plot illustrates the output score of a binary classifier as one input feature varies monotonically.
The output score threshold at which the prediction flips is indicated by the dashed line.
The blue circles mark threshold crossings, each of which witnesses a local \textit{robustness} violation.
The set of three red crosses marks a glitch: a short input interval in which the score
jumps from one prediction to another and then jumps back.
}
\label{fig:glitch versus robustness concept}
\end{figure}

We propose a general definition of glitches, where the underlying model is
treated as a black-box function---capable of representing models including
tree-ensembles, neural networks, and beyond.
Glitches are defined as a witness-based property: a glitch (witness) is an ordered input triple such that the outputs of the two extreme points are close to each other, but the output of the center point is substantially different.
This definition applies to any model whose input and output domains are ordered
metric spaces. 
Moreover, the definition is quantitative: the severity of a glitch indicates the distance between the output of the center point from that of the extreme points.

We do the algorithmic study of searching for glitches in white-box tree ensemble models.
These models are natural for studying glitches, not only because they are widely
used~\cite{friedman2000additive,DBLP:conf/nips/ChenZS0BH19,ghosh2022comparative},
but also due to their susceptibility to exhibit glitch-like output oscillations.
To understand why, we need to understand how they work.
While a single decision-tree partitions the input space into different output
classes, a tree ensemble combines the outputs of all trees, so that the input
space is effectively divided into exponentially many
regions~\cite{bengio_2009_deep_architectures, Bengio2010_decision_trees}. This
architecture is a major boost to the expressive power, but it also means that
along a feature slice, the ensemble output could potentially jump through multiple
adjacent regions in close succession, causing the output to oscillate as the
feature moves monotonically. 
The issue is exacerbated by the non-smoothness of the output, implying that the severity of these oscillations could be unbounded.
Glitches are precisely the high-magnitude instances of such oscillations, and pairwise metrics such as robustness cannot represent them accurately.

Our contributions are threefold. \textit{First}, we provide a general, model-agnostic formalization of glitches described above, applicable to any model over ordered, metric input and output domains. 
\textit{Second}, we prove that verifying whether an ensemble contains a glitch is NP-complete, and remains NP-hard even for trees of depth 4. While the NP-completeness of verifying robustness violations of tree ensembles has been established~\cite{DBLP:conf/icml/KantchelianTJ16}, the reasoning breaks for glitches, and our proof is substantially different and more involved. \textit{Third}, we present an encoding of the verification problem as a mixed-integer linear programing (MILP) that can be solved using state-of-the-art MILP solvers like Gurobi~\cite{gurobi}. We implemented this pipeline in our tool called \ourtool{}. Not only does our experiments demonstrate the superior scalability of our method---scaling up to 1000 trees with tree-depth 8 and hundreds of features, but also they uncover glitches---and potential anomalies---in many existing models for which robustness verification falls short.

\subsection*{Related Work}\label{sec:rel}

AI verification asks whether a trained model satisfies a behavioral
specification, typically by searching for an input, or a small set of inputs,
that witnesses a violation. A range of such specifications and their
verification techniques have been studied in
the literature, which we recall below. 

Robustness requires that a small change in the input should not cause the output
to deviate~\citep{croce2020robustbench,chen2021learning,ruan2019global,leino2021globally,zhong2021understandinglocalrobustnessdeep,MarabouToolPaper,deeppoly,betacrown,
DBLP:conf/icml/KantchelianTJ16,DBLP:conf/aaai/EinzigerGSS19,DBLP:conf/nips/ChenZS0BH19,Zhang2020AnEA}.
There are two main variants: Local robustness asks whether the model output is invariant within a small neighborhood of a given input, while global robustness asks for guarantees over an entire input domain or specification~\citep{croce2020robustbench,chen2021learning,ruan2019global,leino2021globally}.
Robustness verification has been studied extensively for neural networks~\citep{zhong2021understandinglocalrobustnessdeep,MarabouToolPaper,deeppoly,betacrown} and for tree ensembles~\citep{DBLP:conf/icml/KantchelianTJ16,DBLP:conf/aaai/EinzigerGSS19,DBLP:conf/nips/ChenZS0BH19,Zhang2020AnEA}.
A glitch strengthens the robustness specification: every glitch contains a robustness counterexample --- a pair of neighboring inputs with disagreeing outputs --- but the converse fails: a robustness counterexample need not form part of any glitch, since a glitch additionally requires the deviation to be non-monotone and confined to a small interval. Our verification procedure builds on the mixed-integer encoding of tree ensembles introduced by \citep{DBLP:conf/icml/KantchelianTJ16} for robustness, extended here to the three-point witness structure that a glitch requires.

Monotonicity requires a monotonic change in the input (according to a given order)
should cause only monotonic changes in the
output~\citep{ben2009adding,sharma2020higher,DBLP:conf/icml/DevosMD21}.
Monotonicity is important in domains where certain features have known directional semantics, such as income, risk scores, or clinically meaningful measurements~\citep{ben2009adding,sharma2020higher}.
Prior works have studied monotonicity testing and verification for machine learning models, including tree ensembles and related verification frameworks~\citep{sharma2020higher,DBLP:conf/icml/DevosMD21}.
A natural objection is that glitches could simply be trained away by enforcing
monotonicity. This is problematic for two reasons. First, monotonicity is only meaningful
for the few features that carry a known directional semantics; most features
have no such direction, but the discussion of glitches is still valid and relevant. Second, even on features for which
monotonicity is declared, current constrained training does not reliably enforce
it: training XGBoost with its monotonicity constraints, we find that the
resulting models remain non-monotone on many of the constrained features, and
glitches persist on those features (Section~\ref{sec:section5}). 
Aggressively enforcing monotonicity everywhere could preclude glitches, but at the cost of a model that is no longer useful; glitch verification instead identifies the localized violations that remain in models trained for the accuracy–monotonicity trade-off that practitioners actually use.

Sensitivity verification for decision tree ensembles studies whether changes to specified sensitive features can affect the model decision~\citep{ahmadsensitivity}.
This is close in spirit to robustness, except that selected features may be varied more freely.
Glitch verification differs because it asks not merely whether the output can change, but whether it changes and then returns over a small ordered interval.
More broadly, prior work has studied both training-time methods for robust tree-based models~\citep{DBLP:journals/datamine/CalzavaraLTAO20,DBLP:conf/icml/ChenZBH19} and post-training verification of robustness, sensitivity, and other logical properties~\citep{DBLP:conf/icml/KantchelianTJ16,DBLP:conf/nips/ChenZS0BH19,DBLP:conf/icml/DevosMD21,DBLP:conf/aaai/EinzigerGSS19,DBLP:conf/cp/IgnatievCSHM20}.
Our work belongs to the latter line.

Across these lines of work, the specifications that have been formalized for
tree ensembles---robustness, monotonicity, sensitivity---are each
witnessed by a single pair of inputs, and verification tools are built around
such pairwise witnesses. None isolates a localized, non-monotone oscillation: a
bounded interval over which the output jumps back and forth. This is the gap
we address.

\section{Motivating Case Studies}

We present two case studies to motivate the need for dedicated verification tools to identify glitches in decision-makers' outputs.
Both case studies use datasets and pre-trained models from the literature.
The first one shows that glitches are more fine-grained patterns as compared to robustness and monotonicity violations, i.e., tools designed to detect robustness and monotonicity violations report many events that are not glitches.
The second case study shows that existing real-world AI models designed for critical tasks do contain glitches, and these glitches were confirmed as clinically suspicious  by independent domain experts. More experiments are reported in Sec.~\ref{sec:section5}.

\medskip
\noindent\textbf{Case Study~I: Robustness, Monotonicity, and Glitches.}
We use a publicly available binary classifier~\citep{DBLP:conf/nips/ChenZS0BH19} trained on the ijcnn1
dataset.\footnote{Source url: \url{https://github.com/zoq/datasets/tree/master/ijcnn1}}
Let $f\colon \mathbb{R}^d\to \{0,1\}$ represent the classifier.
For every training input $x$, we write $\mathit{label}(x)\in \{0,1\}$ for its label.
We compare three kinds of events in this model: local robustness violations, monotonicity violations, and glitches.
When discussing whether a local robustness violation is anticipated by the data, we compare it with the labels of nearby training points, assuming the training dataset is free of outliers.

\begin{figure}[t]
  \centering
    \pgfkeys{/pgf/number format/.cd,fixed}
    \begin{tikzpicture}
        \begin{axis}[scale only axis,           width=0.3\linewidth,               height=0.18\linewidth,
            scaled ticks=false,
            x label style={at={(axis description cs:0.5,0)},anchor=north},
            y label style={at={(axis description cs:0.4,1.1)},anchor=south},
            xlabel={(a)},
            ylabel={Out},
            axis lines = left,
            enlarge x limits = true,
            enlarge y limits = true,]
        \addplot table [x=f18, y=prob, col sep=comma,mark=none] {data/ijcnn_robust_252_f18_model_line.csv};
    \addplot[
            color=red,
            mark=*,
            only marks,
            mark size=1.2pt
        ] 
        table [x=f18, y=out, col sep=comma] {data/ijcnn_robust_252_f18_traindt.csv};

    \addplot[
            color=black,
            mark=*,
            only marks,
            mark size=1.2pt
        ] 
        table [x=f18, y=out, col sep=comma] {data/ijcnn_robust_252_f18_reference.csv};
        
       \draw[dashed]	(axis cs:0.45,0.5)	--	(axis cs:0.55,0.5)	node[pos=0.6,above]	{$\text{prediction} = 1$}	node[pos=0.6,below]	{$\text{prediction} = 0$};
        \draw[<->]	(axis cs:0.48,0.35)		--	(axis cs:0.48,0.65);
    \end{axis}
    \end{tikzpicture}
        \pgfkeys{/pgf/number format/.cd,fixed}
    \begin{tikzpicture}[spy using outlines={circle,black,magnification=2,size=1cm, connect spies}]
        \begin{axis}[scale only axis,           width=0.3\linewidth,               height=0.18\linewidth,
            scaled ticks=false,
            x label style={at={(axis description cs:0.5,0)},anchor=north},
            y label style={at={(axis description cs:0.4,1.1)},anchor=south},
            xlabel={(b)},
            ylabel={},
            axis lines = left,
            enlarge x limits = true,
            enlarge y limits = true,]
        \addplot table [x=f18, y=prob, col sep=comma,mark=none] {data/ijcnn_robust_9082_f18_model_line.csv};
    \addplot[
            color=red,
            mark=*,
            only marks,
            mark size=1.2pt
        ] 
        table [x=f18, y=out, col sep=comma] {data/ijcnn_robust_9082_f18_traindt.csv};
        
        \addplot[mark=triangle*,yellow,draw=black,mark size=2pt]    coordinates {(0.4703,0.757)};
        \node   (b) at    (axis cs:0.467,0.75)   {\footnotesize $p^-$};
        \addplot[mark=triangle*,yellow,draw=black,mark size=2pt]   coordinates {(0.4708,0.4875)};
        \node   (c) at    (axis cs:0.468,0.4)   {\footnotesize $p$};
		\addplot[mark=triangle*,yellow,draw=black,mark size=2pt]    coordinates {(0.4722,0.567)};
        \node  (d) at    (axis cs:0.48,0.6)   {\footnotesize $p^+$};
        
        \coordinate	(spypoint)	at	(axis cs:0.474,0.6);
        \coordinate (magnifying)	at	(axis cs:0.52,0.6);
        
        \draw[dashed]	(axis cs:0.45,0.5)	--	(axis cs:0.55,0.5);

    \end{axis}
    
    \spy [black,size=2.5cm]	on	(spypoint)		in node[fill=white]	at	(magnifying);
    \end{tikzpicture}
  \caption{
  	Comparing robustness and glitch events in a pre-trained binary classifier~\citep{DBLP:conf/nips/ChenZS0BH19}.
  	The X-axes vary feature f18 while all other features are fixed.
  	The Y-axes show the model output. Values above $0.5$ are predicted as label $1$, and values below $0.5$ as label $0$.
  	Red dots are nearby training points.
  	Panel (a) shows an anticipated local robustness violation caused by nearby training points with different labels.
  	Panel (b) shows a glitch: as f18 increases, the three marked points satisfy
  	$p^-=(0.4703,0.757)$, $p=(0.4708,0.4875)$, and $p^+=(0.4722,0.567)$.
  	Thus $p^-$ and $p^+$ are predicted as label $1$, while the intermediate point $p$ is predicted as label $0$.
  	The glitch event therefore contains two opposite threshold crossings in a short interval, not just a single pairwise output change.
    }\label{fig:spike1}
  \end{figure}

\noindent
{\bf Robustness:}
For a given $\epsilon>0$, the classifier $f$ is called locally robust around an input $x\in \mathbb{R}^d$ if for every $x'\in \mathbb{R}^d$ with $\|x-x'\|\leq \epsilon$, $f(x)= f(x')$~\cite{bastani2016measuring,DBLP:conf/nips/ChenZS0BH19}.
We use the existing \texttt{treeVerification}~\citep{DBLP:conf/nips/ChenZS0BH19} tool to find robustness violations in the model in the $\epsilon$-neighborhoods of test data points.

A robustness violation around the test input $x$ is \emph{anticipated} if there exists a pair of training inputs $y,y'$ in the $\epsilon$-neighborhood of $x$ such that $\mathit{label}(y)\neq \mathit{label}(y')$, and \emph{unanticipated} otherwise.
An anticipated robustness violation is not necessarily undesirable as it may reflect genuine label disagreement within the $\epsilon$-neighborhood.
For instance, in a grading scenario, students with grades $\geq 8$ are promoted while those with grades $<8$ are not, and points straddling the threshold 8 cannot be locally consistent.
One illustration of anticipated robustness is Fig.~\ref{fig:spike1}(a).
This is a single-boundary phenomenon: a pair of nearby points receives different predictions. However, there is no short interval over which the model crosses the boundary and then crosses back. In our experimental evaluation, we also report on the analysis of robustness violations and show that most of these robustness violations  are anticipated. Consequently, using robustness violations as a proxy for measuring inconsistencies may yield a substantial false positive rate due to the considerable number of anticipated cases.

\noindent{\bf Monotonicity:}
The classifier $f$ is monotonic with respect to a given feature dimension $i$ if for every $x,x'\in \mathbb{R}^d$ with $x_i>x'_i$ and $x_j=x_j'$ for every $j\neq i$, it holds that $f(x)\geq f(x')$.
Monotonicity testing and verification have been studied for machine-learning models, including tree ensembles and related verification frameworks~\citep{sharma2020higher,DBLP:conf/icml/DevosMD21}.
However, monotonicity focuses on a different observation from glitch detection.
A monotonicity violation is observed when, for an ordered pair of inputs, outputs go in the wrong direction.
A glitch requires an ordered triple whose outputs go down and then up, or up and then down, within a small interval.
Thus, a monotonicity checker may identify a decreasing pair in Fig.~\ref{fig:spike1}(b), but it does not by itself isolate the full return pattern $1\to0\to1$.

\noindent      
{\bf Glitches:}        
We now turn to glitches.
For this binary classifier, $f$ has a glitch in dimension $i$ around an input $x$ if there are two nearby inputs $x^-,x^+$ with $\|x^+-x^-\|\leq \epsilon$ such that $x^-_i<x_i<x^+_i$, $x^-_j=x_j=x^+_j$ for every $j\neq i$, and $f(x^-)=f(x^+)$ but $f(x)\neq f(x^-)$.
A more general definition for real-valued model outputs is given in Def.~\ref{def:glitch}.

In Fig.~\ref{fig:spike1}(b), the three marked points form a glitch:
The scores satisfy $0.4703<0.4708<0.4722$, while all other features are fixed.
Thus the model undergoes a $1\to0\to1$ return pattern over a width of $0.0019$ in feature f18.
A robustness verifier can detect either neighboring threshold crossing, and a monotonicity verifier can detect an ordered pair with decreasing output; the glitch witness records the stronger fact that both changes occur together as a localized oscillation.
The anticipated/unanticipated distinction is orthogonal to this definition. It concerns whether nearby training labels already disagree, whereas glitch detection concerns the shape of the model output along the slice.

\begin{figure}
    \centering
    \begin{tikzpicture}
    \begin{axis}[
        width=7.0cm,
        height=4.2cm,
        scaled ticks=false,
        xlabel={MCP},
        ylabel={Out},
        xlabel style={font=\small},
        ylabel style={font=\small},
        tick label style={font=\small},
        axis lines=left,
        xmin=0.039,
        xmax=0.091,
        ymin=0.25,
        ymax=0.86,
        enlarge x limits=false,
        enlarge y limits=false,
    ]

        \addplot[
            blue!85!black,
            very thick,
            mark=none
        ] coordinates {
            (0.039,0.7999)
            (0.05069,0.7999)
            (0.05069,0.2945)
            (0.05259,0.2945)
            (0.05259,0.7126)
            (0.0539,0.7126)
            (0.0539,0.82)
            (0.091,0.82)
        };

        \addplot[
            mark=triangle*,
            yellow,
            draw=black,
            mark size=3pt
        ] coordinates {(0.05069,0.7999)};
        \node at (axis cs:0.0475,0.79) {\small $p^-$};

        \addplot[
            mark=triangle*,
            yellow,
            draw=black,
            mark size=3pt
        ] coordinates {(0.05259,0.2945)};
        \node at (axis cs:0.058,0.31) {\small $p$};

        \addplot[
            mark=triangle*,
            yellow,
            draw=black,
            mark size=3pt
        ] coordinates {(0.0539,0.7126)};
        \node at (axis cs:0.060,0.72) {\small $p^+$};

        \draw[dashed]
            (axis cs:0.04,0.5) -- (axis cs:0.09,0.5)
            node[pos=0.76,above,align=center,font=\small]
                {prediction =\\ malignant}
            node[pos=0.76,below,align=center,font=\small]
                {prediction =\\ benign};

        \draw[<->]
            (axis cs:0.063,0.45) -- (axis cs:0.063,0.55);

    \end{axis}
\end{tikzpicture}
    \caption{Glitch in the breast-cancer prediction model along the MCP feature.}
    \label{fig:breast_cancer_glitch}
\end{figure}
\noindent\textbf{Case Study~II: Glitches in Breast Cancer Prediction Model.}
Using XGBoost~\citep{DBLP:conf/kdd/ChenG16} and a public dataset~\citep{kagglebreastCancer}, 
we trained a tree ensemble model for detecting the likelihood of malignancy of a breast mass from images obtained via the fine needle aspiration technique, a standard procedure used to obtain a cell sample from a suspicious lump.

In this benchmark, there is a feature variable called ``mean concave points'' (MCP), which indicates the average number of concave points in the cell boundaries---a critical feature for describing deformities in cell structures, where higher deformities usually indicate a higher likelihood of malignancy~\citep{breastCancerPaper}.
We discovered a glitch in the MCP feature (Fig.~\ref{fig:breast_cancer_glitch}), where for a fixed arrangement of the remaining feature values , the model's predicted likelihood of malignancy abruptly oscillates as MCP increases.
The three points $p^-$, $p$, and $p^+$ have increasing MCP values $0.05069<0.05259<0.0539$, while the outputs are $0.7999$, $0.2945$, and $0.7126$, respectively.
Thus the model crosses the decision threshold from malignant to benign and then back to malignant over a very short MCP interval.

A pairwise robustness check around this region can report a nearby threshold crossing, and a monotonicity check with respect to MCP can report the decrease from $p^-$ to $p$.
Neither pairwise event records the complete high--low--high structure over the interval.
This is the additional information exposed by the glitch event. It identifies a narrow decision island along a clinically meaningful feature.
Three oncologists independently judged this particular MCP pattern to be clinically suspicious when shown the fixed feature values and the three model outputs.
This example illustrates that glitch witnesses are worth verifying as they expose localized oscillatory structures that may be missed or under-specified by standard pairwise robustness and monotonicity checks, and they provide concrete candidates for domain-level inspection.

\section{Formalizing Glitches in AI Decision-Makers}\label{sec:section3}

We model AI decision-makers as functions of the form $f\colon\mathbb{T}^m\to \mathbb{T}$, where $m>0$ and $\mathbb{T}$ is any ordered set equipped with a known order ``$\leq$'' and a distance metric ``$d(\cdot,\cdot)$,'' assigning a non-negative distance to any two points in $\mathbb{T}$.
(They are not required to be the same across all input dimensions and the output, though we use the same notations ``$\leq$'' and  ``$d(\cdot,\cdot)$'' for simplicity.)
For example, a boolean classifier with Euclidean feature space can be modeled as $f\colon \mathbb{R}^m\to \mathbb{B}$, where the set $\mathbb{R}$ has the usual order ``$\leq$'' over real numbers and the Euclidean metric ``$\|\cdot\|$,'' and the set $\mathbb{B}$ can be equipped with the order ``$0\leq 1$'' and the metric ``$d(0,1) = 1$.''
Similarly, a decision tree (to be formally defined in Sec.~\ref{sec:section2}) with Euclidean feature space can be modeled as $f\colon \mathbb{R}^m\to \mathbb{R}$.

Before formalizing glitches, we need to extend the ordering ``$\leq$'' and the metric $d(\cdot,\cdot)$ over $\mathbb{T}$ to an ordering and a metric over $\mathbb{T}^m$:
For a given dimension $i\in [1;m]$, we introduce the ordering ``$\leq_i$'' over $\mathbb{T}^m$, such that $x\leq_i y$ 
if $x_i\leq y_i$, while $x_j=y_j$ for every $j\neq i$. 
Moreover, $x<_i y$ if $x\leq_i y$ and $x\neq y$.
For example, $(1,2,3)\leq_2 (1,3,3)$ and $(1,2,3)\leq_3 (1,2,4)$.
Clearly ``$\leq_i$'' is a partial ordering: e.g., $(1,2,3)$ and $(2,1,3)$ are not comparable using any of $\leq_1$, $\leq_2$, or $\leq_3$. 
However, any two vectors that only differ in their $i$-th dimension can always be compared using ``$\leq_i$.''
For a pair of points $x,y$ with $x\leq_i y$,  $d(x,y)$  denotes $d(x_i,y_i)$.

\begin{df}[Glitch]\label{def:glitch}
    Let $f\colon\mathbb{T}^m\to \mathbb{T}$ be a given decision-maker and $\alpha\in \mathbb{R}_{>0}$ be a given constant. 
    Let $(x^-,x,x^+)$ be an input triple  with 
    \begin{align}\label{eq:glitch:ordering of inputs}
    x^-\leq_i x\leq_i x^+
    \end{align}
    for some dimension $i$.
    The triple $(x^-,x,x^+)$ will be called an $\alpha$-glitch of $f$ in the dimension $i$ if the following two conditions hold:
    \begin{align}\label{eq:glitch:sharpness}
        \frac{\min\{d(f(x),f(x^-)), d(f(x),f(x^+))\}}{d(x^-,x^+)} \geq \alpha 
    \end{align}
    and
    \begin{align}\label{eq:glitch:non-monotonicity}
             f(x^-) > f(x) \, \land \, f(x)<f(x^+) 
             \qquad \text{ or } \qquad 
             f(x^-) < f(x) \, \land \, f(x) > f(x^+).
         \end{align}
\end{df}

Eq.~\eqref{eq:glitch:sharpness} imposes $\alpha$ as the minimum abruptness of the output fluctuations as we travel from $x^-$ to $x$ to $x^+$ along dimension $i$: the smaller the distance $d(x^-,x^+)$ is, and the larger each of the jumps $d(f(x^-),f(x))$ and $d(f(x),f(x^+))$ is, the more abrupt is the fluctuation.
Eq.~\eqref{eq:glitch:non-monotonicity} formalizes the requirement that either there is a drop from $f(x^-)$ to $f(x)$ followed by a rise from $f(x)$ to $f(x^+)$ (first condition)---forming a ``canyon''-shaped glitch, or there is a rise from $f(x^-)$ to $f(x)$ followed by a drop from $f(x)$ to $f(x^+)$ (second condition)---forming a ``hill''-shaped glitch.
The glitches found in the motivating examples from Fig.~\ref{sec:intro} were canyon-shaped.

We formalize \emph{magnitudes} of glitches as follows.
Let $(x^-,x,x^+)$ be an $\alpha$-glitch of $f$ for a given $\alpha$.
Clearly, $(x^-,x,x^+)$ is also an $\alpha'$-glitch  for every $\alpha'\leq \alpha$.
The \emph{magnitude} of the glitch $(x^-,x,x^+)$ is the supremum of the set of $\alpha''$ for which $(x^-,x,x^+)$ is an $\alpha''$-glitch of $f$.
Intuitively, the higher the magnitude of a glitch is, the more abrupt are the fluctuations and more ``pointed'' it looks visually.

\begin{remark}[Multi-dimensional glitches]
    We define one-dimensional glitches, i.e., the three input points $x^-$, $x$, and $x^+$ in Def.~\ref{def:glitch} differ only in a single dimension (denoted $i$).
    In theory, we may consider multi-dimensional glitches by letting the input to simultaneously vary along multiple dimensions, by defining a suitable ordering of inputs that would replace ``$\leq_i$'' in \eqref{eq:glitch:ordering of inputs} in a well-defined manner when $x^-$, $x$, and $x^+$ may differ along multiple dimensions.
    We leave this generalization open for future, and focus on the simplest model of one-dimensional glitches, which are already abundantly found in our experiments on available models.
\end{remark}

\smallskip
\noindent\textbf{Glitches, Lipschitz Continuity, Robustness, and Fairness.}
Lipschitz continuity has close connection to a range of reliability metric of AI classifiers, including adversarial robustness~\citep{zuhlke2024adversarial}, global robustness~\citep{leino2021globally}, and individual fairness~\citep{dwork2012fairness}, and in all these applications, typically small Lipschitz constants are desirable.
Lipschitz continuity is usually defined on Euclidean spaces, for which, we assume that the decision-maker has the form $f\colon \mathbb{R}^m\to \mathbb{R}$.
The decision-maker $f$ is called (globally) \emph{Lipschitz continuous} if a small change in the input causes a proportionately small change in the output, i.e., if 
$L\coloneqq \sup_{x,y}(\|f(x)-f(y)\|)/(\|x-y\|) < \infty$, in which case the constant $L$ is called the \emph{Lipschitz constant} of $f$.

Lipschitz constant controls how fast the output increases or decreases, whereas glitch magnitude controls how fast the output moves from an increasing trend to a decreasing trend or vice versa.
In the following we draw a formal connection between Lipschitz constants and glitch magnitudes.

\begin{restatable}{proposition}{lipschitz}
\label{prop:lipschitz continuity and glitches}
    Let $f$ be a Lipschitz continuous decision-maker with the Lipschitz constant $L$.
    The magnitude of every glitch of $f$ is at most $L/2$.
\end{restatable}

\begin{proof}
    Suppose $x^+$ and $x^-$ are any two inputs.
    We need to find $x$ and the \emph{maximum} $\alpha$ such that \eqref{eq:glitch:ordering of inputs}, \eqref{eq:glitch:sharpness}, and \eqref{eq:glitch:non-monotonicity} are satisfied, while assuming that $\|f(y)-f(z)\|\leq L\|y-z\|$ for every inputs $y,z$ (by Lipschitz continuity).
    Suppose, $x$ is located between $x^-$ and $x^+$ such that $(\|x^+-x\|)/(\|x^+-x^-\|) = \lambda$ for a $\lambda\in [0,1]$ whose value is to be determined, so that $(\|x-x^-\|)/(\|x^+-x^-\|) = (1-\lambda)$.
    Then $\|f(x^+)-f(x)\| \leq L\lambda\|x^+-x^-\|$ and $ \|f(x)-f(x^-)\| \leq L(1-\lambda)\|x^+-x^-\|$, implying that $\min\{\|f(x^+)-f(x)\|, \|f(x)-f(x^-)\| \}/(\|x^+-x^-\|) \leq L\cdot\min\{\lambda,1-\lambda\}$.
    It is easy to check that $\max_{\lambda\in [0,1]}L\cdot\min\{\lambda,1-\lambda\} = L/2$, which is the maximum magnitude of the glitch $(x^-,x,x^+)$.
\end{proof}

Therefore, every Lipschitz continuous---aka robust---decision-maker with small Lipschitz constant can only have glitches with small magnitudes. 
However, the other direction is not generally true.
For example, the function $y=e^x$ is monotonic in $x$ and therefore does not have any glitch in its entire domain (Eq.~\eqref{eq:glitch:non-monotonicity} will never be satisfied), i.e., the magnitude of every glitch is $0$.
However, the function is not even Lipschitz continuous, i.e., there is no \emph{finite} Lipschitz constant.

This comparison suggests that the existing notion of Lipschitz continuity or robustness are not adequate for modeling and analyzing glitches.
Besides, many AI decision-makers are either not (globally) Lipschitz continuous, or, even if they are, finding the Lipschitz constant may be a challenging problem.
Therefore, Prop.~\ref{prop:lipschitz continuity and glitches} cannot always be used to rule out the presence of high-magnitude glitches, and we need specialized tools to detect and analyze glitches in AI models.

\section{Verifying Glitches in Decision Tree Ensemble Models}\label{sec:section2}

Towards the algorithmic study of glitches in AI models, as a first step, we consider widely used decision tree ensemble models \citep{friedman2000additive,DBLP:conf/nips/ChenZS0BH19}, which are piecewise linear functions with sharp discontinuities between the linear ``pieces,'' making these models globally non-robust.
Similar discontinuities do not exist in AI models such as feed-forward neural networks.
Furthermore, decision-tree ensembles can be non-monotonic in general, for example if they are used for image classification tasks~\citep{ghosh2022comparative}.
These features create the perfect opportunity for the study of glitches.

We briefly recall the definitions of decision trees and their ensembles.
Fix a set of feature variables $V=\{ v_1,\ldots,v_m \}$ which take values over $\mathbb{R}$.
A decision tree $\tree$ over $V$ implements a decision-maker $\sem{\tree}\colon \mathbb{R}^m\to \mathbb{R}$. 
Syntactically, it  is a rooted binary tree whose internal nodes  are labeled with predicates of the form $v<\eta$, where $v\in V$ and $\eta$ is a rational constant, and leaf nodes are labeled with real constants.
For a given decision tree $\tree$, $\tree.\intnodes$ and $\tree.\leaves$  respectively denote its  internal nodes and leaves.
 For each internal node $\node$, $\node.\pred$ denotes the predicate label of $\node$, and has two children  $\node.\true$ and $\node.\false$, respectively.
Let $x\in\mathbb{R}^m$ be an arbitrary feature assignment for the variables in $V$, where $x_i$ denotes the value of $v_i$ for $i\in [1;m]$.
The assignment $x$ generates a unique sequence of nodes $n_1\ldots n_k$, called a \emph{path}, in $\tree$, such that (a)~$n_1$ is the root of $\tree$, (b)~$n_k\in \tree.\leaves$, and (c)~for every $j\in [1;k-1]$, assuming $n_j.\pred = ``v_i< \eta"$ for some $i\in [1;m]$ and constant $\eta$, the successor $n_{j+1}$ is $n_j.\true$ if $x_i< \eta$ and is $n_j.\false$ otherwise. The \emph{output} of $\tree$ for the \emph{input} $x$ is given by $\sem{\tree}(x)=n_k.\const$. 

A \emph{decision tree ensemble} $\M$ over the feature variables $V$ is a finite set $\{\tree_1,\ldots,\tree_l\}$ of decision trees over $V$, implementing the decision-maker $\sem{\M}\colon \mathbb{R}^{|V|}\to \mathbb{R}$ defined as $\sem{\M}\colon x\mapsto \sum_{j=1}^l \sem{\tree_j}(x)$, and is called 
 the \emph{output} of $\M$ for  input $x$.

\subsection{Problem Statement and Complexity Results}\label{sec:stmt}

We pose three fundamental algorithmic questions, namely deciding the existence of glitches and searching for glitches with the highest magnitude.

\begin{problem}[$\TEGLITCHi$]\label{prob:dim}\hfill\\
    \emph{Input}: a decision tree ensemble $\M$, a constant $\alpha>0$ and a dimension $i$.\\
    \emph{Output:} a glitch  in the dimension $i$ with magnitude $> \alpha$, or output that such a glitch does not exist in dimension $i$.
\end{problem}

\begin{problem}[$\TEGLITCH$]\label{prob:search glitches}\hfill\\
    \emph{Input}: a decision tree ensemble $\M$, a constant $\alpha>0$.\\
    \emph{Output:} a glitch  with magnitude  $> \alpha$, or output that such a glitch does not exist.
\end{problem}

\begin{problem}[$\TEGLITCHmax$]
\label{prob:search sharpest glitches}\hfill\\
    \emph{Input}: a decision tree ensemble $\M$.\\
    \emph{Output}: a glitch whose magnitude is at least as large as every other glitch of $\M$ or output that no glitch exists.
\end{problem}

We establish upper and lower complexity bounds for the verification problems $\TEGLITCHi$ and $\TEGLITCH$, whose proofs depend on the following result  of independent interest.
The following proposition essentially narrows down the circumstances under which glitches would exist in tree ensemble models.

\begin{restatable}{proposition}{glitchlocation}
\label{prop:glitch location}
    Suppose $\M$ is a decision tree ensemble, and $(x^-,x,x^+)$ is a glitch of $\M$ in a given dimension $i$.
    Then there exists a pair of distinct trees in $\M$ which have internal nodes respectively with predicates $v_i\leq a$ and $v_i\leq b$, such that $x^-_i\leq a<x_i\leq b< x^+_i$.
\end{restatable}

\begin{proof}
    For the given dimension $i$, the predicates $v_i\leq \eta_i$ of the trees, for various $\eta_i$, partitions the domain of $v_i$.
    If for contradiction's sake we assume that the claim is false, then it would mean that any two values among $x^-_i,x_i,x^+_i$ would fall in the same partition created by the predicates on $v_i$.
    Suppose, without loss of generality, that $x^-_i$ and $x_i$ are within the same partition element.
    Since the variables $j\neq i$ are fixed in $x^-,x,x^+$, therefore, $x^-$ and $x$ would activate the exact same internal nodes in all the trees, and we would obtain $\M(x^-)=\M(x)$, which would mean that the magnitude of the glitch would be $0$---a contradiction.
\end{proof}

With this auxiliary result, we are able to establish the following tight complexity bound for the decision problems $\TEGLITCHi$ and $\TEGLITCH$.

\begin{theorem}\label{thm:NP membership}
    $\TEGLITCHi$ and $\TEGLITCH$ are in NP.
\end{theorem}

\begin{proof}
    It is easy to see that both problems have certificates (the glitch) that can be checked in polynomial time.
\end{proof}

We move on to lower bounds. 

\begin{restatable}{theorem}{coBuchischeduling}\label{thm: NP-hardness}
   $\TEGLITCHi$ and $\TEGLITCH$ are NP-hard.
\end{restatable}

\begin{figure}
    \centering
    \begin{tikzpicture}[every node/.style={scale=0.7},node distance=0.3cm]
        \node[state]    (r)    at  (0,0)   {$r\leq 0.5-\varepsilon$};
        \node[state]    (a)     [right=of r]   {$v_a\leq 0.5$};
        \node[state]    (b)     [right=of a]   {$v_b\leq 0.5$};
        \node[state]    (c)     [right=of b]   {$v_c\leq 0.5$};

        \node[state,rectangle]    (rr)     [above right=of r]   {$0$};
        \node[state,rectangle]    (al)     [above right=of a]   {$1$};
        \node[state,rectangle]    (br)     [above right=of b]   {$1$};
        \node[state,rectangle]    (cl)     [ below right=of c]   {$1$};
        \node[state,rectangle]    (cr)     [above right=of c]   {$0$};

        \path[-]    
            (r) edge  (rr)
            (a) edge    (b)
            (b) edge    (br)
            (c) edge    (cr);

        \path[dashed]    
            (r) edge  (a)
            (a) edge    (al)
            (b) edge    (c)
            (c) edge    (cl);
    \end{tikzpicture}
    \begin{tikzpicture}[every node/.style={scale=0.7},node distance=0.3cm]
        \node[state]    (r)    at  (0,0)   {$r\leq -0.5$};
        \node[state]    (a)     [right=of r]   {$v_a\leq 0.5$};
        \node[state]    (b)     [right=of a]   {$v_b\leq 0.5$};
        \node[state]    (c)     [right=of b]   {$v_c\leq 0.5$};

        \node[state,rectangle]    (rl)     [above right=of r]   {$0$};
        \node[state,rectangle]    (al)     [above right=of a]   {$1$};
        \node[state,rectangle]    (br)     [above right=of b]   {$1$};
        \node[state,rectangle]    (cl)     [ below right=of c]   {$1$};
        \node[state,rectangle]    (cr)     [above right=of c]   {$0$};

        \path[-]    
            (r) edge  (a)
            (a) edge    (b)
            (b) edge    (br)
            (c) edge    (cr);

        \path[dashed]    
            (r) edge  (rl)
            (a) edge    (al)
            (b) edge    (c)
            (c) edge    (cl);
    \end{tikzpicture}
    \caption{Illustration: reducing the toy 3-CNF-SAT instance $(\lnot a\lor b \lor \lnot c)$ to a tree ensemble for $\TEGLITCHi$. Denoting the (only) clause as $c_1$, LEFT: $T_1$, RIGHT: $T_1'$. For each node, the child connected to it via a solid or dashed edge is its ``true'' or ``false'' child, respectively.}
    \label{fig:np-hardness-construction}
\end{figure}

\begin{proof}
We first show that $\TEGLITCHi$ is NP-hard. 

\smallskip
\noindent\textbf{The problem $\TEGLITCHi$ is NP-hard.}
The reduction is from 3-CNF-SAT which is known to be NP-complete. Let $\varphi=\bigwedge_{i=1}^mC_i$ 
be an instance of 3-CNF-SAT consisting of $m$ clauses $C_1, \dots, C_m$ where each clause is a disjunction 
of 3 literals chosen from a set of variables $Z=\{z_1, \dots, z_n\}$. We construct a 
tree ensemble $\M$ with $2m$ trees such that each clause $C_i$ maps to a unique pair of trees
$T_i, T_i'$ in $\M$. We illustrate our construction in Fig.~\ref{fig:np-hardness-construction} on a toy example.
The maximum depth of any tree in $\M$ is 4, and we have $n+1$ variables 
$\{v_i\}_{i \in [n]} \cup  \{r\}$, where $[n]=\{1,2,\dots,n\}$.  
Here $r$ is a new variable, and we will use it interchangeably to denote both the variable and its assignment; recall that for variables $\{v_i\}_i$ we separately use $\{x_i\}_i$ to denote the respective assignments.
Each non-root node of a tree has the form $v_i \leq 0.5$, and if it evaluates to true, then it 
corresponds to assigning $z_i=\top$ (true), and if it evaluates to false, it corresponds 
to assigning $z_i=\bot$ (false).

Fix a clause $C_i$ of $\varphi$. Both trees $T_i, T_i'$ contain a common subtree 
$T_i''$ defined below. Each level of $T_i''$ has exactly one internal 
node and it corresponds to a literal in the clause $C_i$. 
 A path ends in a leaf with value 0, if it constitutes an  
assignment where $C_i$ evaluates to false, while paths which constitute an assignment 
where $C_i$ evaluates to true end in a leaf whose value is 1. 
Thus, a tree $T_i$ evaluates to 1 iff it represents a satisfying assignment for $C_i$.

    The root of $T_i$ has the predicate $r\leq 0.5-\varepsilon$, for $0<\varepsilon<1/m$, whose ``true'' child is $0$ and ``false'' child is the root of $T_i''$.
    In contrast, the root of $T_i'$ has the predicate $r\leq -0.5$, whose ``true'' child is the root of $T_i''$ and ``false'' child is $0$.
    We obtain $(\M,\alpha{=}m,i{=}r)$ as an input to the $\TEGLITCHi$ problem  and we claim that $\TEGLITCHi$ outputs a glitch iff $\varphi$ is satisfiable.

    [If:] Suppose $\varphi$ is satisfiable. From a satisfying assignment $\nu$ for $\varphi$, 
     we construct an output for $\TEGLITCHi$.
    For each variable in $\{v_i\}_{i \in [n]}$, assign the value $x_i = 0.5$ to $v_i$ if $\nu(z_i)=\top$, and 
    assign $x_i=0$ to $v_i$ if $\nu(z_i)=\bot$.  Let $x$ denote the entire assignment $\{x_i\}_{i \in [n]}$.
    It is easy to observe that the assignment $x$ in each sub-tree $T_i''$ would produce the value $1$.
    We claim that $((x,r{=}-0.5),(x,r{=}0),(x,r{=}0.5{-}\varepsilon/2))$ is an output of $\TEGLITCHi$.
    When $r{=}-0.5$, each of the $T_i'$ trees will output $1$, with the rest outputting $0$, giving us $\M((x,r{=}-0.5))=m$, when $r{=}0.5{-}\varepsilon/2$, each of the $T_i$ trees will output $1$, with the rest outputting $0$, giving us $\M((x,r{=}0.5-\varepsilon/2))=m$, and, when $r{=}0$, all the trees will output $0$, giving us $\M((x,r{=}0)){=}0$.
    Therefore, for $\alpha' = m/(1-\varepsilon/2)$, the triple $((x,r{=}-0.5),(x,r{=}0),(x,r{=}0.5-\varepsilon/2))$ is an $\alpha'$-glitch.
    Clearly, $\alpha'>\alpha =m$.

    [Only if:]
Now suppose $\TEGLITCHi$ outputs an $\alpha'$-glitch in dimension $r$, $
    \big((x,r^-),(x,r),(x,r^+)\big),
$
with $\alpha' > \alpha = m$.
    We show how this leads to a satisfying assignment for $\varphi$.
    First, observe that the nodes in the trees with the variable $r$ create three disjoint regions, namely $R_1 = \{r\mid r\leq -0.5\}$, $R_2=\{r\mid -0.5<r\leq 0.5-\varepsilon\}$, and $R_3=\{r\mid r>0.5-\varepsilon\}$.
    From Prop.~\ref{prop:glitch location}, we infer that $r^-$ and $r^+$ must lie in $R_1$ and $R_3$, respectively, while $r$ lies in $R_2$.
    Therefore, the distance between $r^-$ and $r^+$ is strictly greater than $1-\varepsilon$, and hence, from \eqref{eq:glitch:sharpness}, we obtain $\min\{ |\M((x,r))-\M((x,r^-))|, |\M((x,r^+))-\M((x,r))|\} > \alpha'(1-\varepsilon) > \alpha(1-\varepsilon) = m(1-\varepsilon) > m\cdot (m-1)/m = m-1$, where the last inequality is due to $\varepsilon < 1/m$.
    From our construction, it is impossible for $\M$ to output fractional value, and therefore, the minimum possible differences (for any choice of $x,r^-,r,r^+$) between $\M((x,r^-))$ and $\M((x,r))$, and $\M((x,r))$ and $\M((x,r^+))$ must be $m$.
    Furthermore, the maximum possible differences between $\M((x,r^-))$ and $\M((x,r))$, and $\M((x,r))$ and $\M((x,r^+))$ must also be $m$, because for $r^-$ and $r^+$, respectively, only the $T_i$ and only the $T_i'$ trees are ``activated'' while the rest output zero, giving us the maximum output $m$, while for $r$, all trees output zero.
    Therefore, we conclude that $|\M((x,r)) - \M((x,r^-))| = |\M((x,r^+))-\M((x,r))| = m$.
    It follows from the construction that the choice made in the assignment $x$ must coincide with a 
    satisfying assignment for $\varphi$.

    [Complexity of the reduction:] We only need to argue that the reduction is in polynomial time : this is clear since we have $2m$ trees over $n+1$ variables such that 
    we have a $m$-glitch in dimension $r$ iff the formula is satisfiable.

    \smallskip
\noindent\textbf{The problem $\TEGLITCH$ is NP-hard.}
Consider the above reduction from 3-CNF-SAT to $\TEGLITCHi$.
    Each non-root vertex of each tree in $\M$ had guards of the form $v_i\leq 0.5$, which implies, by Prop.~\ref{prop:glitch location}, that there will not exist any glitch in any dimension other than $r$.
    It follows that $\TEGLITCHi$ outputs a glitch for the input $(\M,\alpha=m,i=r)$ iff $\TEGLITCH$ outputs a glitch for the input $(\M,\alpha)$. 
    Thus, the same reduction for Thm.~$\TEGLITCHi$ is also a reduction from 3-CNF-SAT to $\TEGLITCH$.
\end{proof}

The following complexity bounds follow from Thm.~\ref{thm:NP membership} and \ref{thm: NP-hardness}.

\begin{restatable}{corollary}{complexity}
	\label{thm:np-completeness}
	$\TEGLITCHi$ and $\TEGLITCH$ are NP-complete.
\end{restatable}

Surprisingly, the proof of Thm.~\ref{thm: NP-hardness} implies that the problems remain NP-hard, even if we fix the depths of the trees to a constant $d\geq 4$ specified in unary. 
Firstly, our proof already establishes NP-hard when $d=4$.
Secondly, for $d>4$, our reduction can be modified by adding $4-d$ ``dummy nodes'' in the trees to increase their depths to $d$ without affecting the outputs.

\subsection{Verification Algorithms via MILP Encoding}\label{sec:algo}

We now give an outline of the algorithms for $\TEGLITCHi$, $\TEGLITCH$, and $\TEGLITCHmax$, as described in Prob.~\ref{prob:dim}, \ref{prob:search glitches}, and \ref{prob:search sharpest glitches}, respectively, where we will encode the problems in mixed-integer linear programming (MILP).
We will explain only the encoding of the problem $\TEGLITCHmax$ in detail; the other two will follow similarly as will be discussed in the end. 

\smallskip
\noindent\textbf{Variables.}
Suppose we are given the tree ensemble $\M=\{\tree_1,\ldots,\tree_r\}$ over the set of variables $V$.
For each tree $\tree \in \M$, Let $\Le_\tree$ be the set of leaves and $\I_\tree$ be the set of internal (non-leaf) nodes.
Each internal node $u \in \I_\tree$ tests the feature $i \in \M$ against a sorted threshold index $k(u) \in \{1, \ldots ,K_i\}$
Let $\Le_u.\true$ (resp. $\Le_u.\false$) denotes the leaves reachable from the true (resp\. false) branch at node u.
Let $l\in \Le_\tree$ carries the leaf value $(v_{tl} \in \mathbb{R})$.
We introduce the variables for MILP encoding.
(a) we declare a \emph{binary variable} $p_{i,k} \in \{0,1\}$  which equals to 1 iff the node which have the feature i and threshold of index k  is true.
(b) we declare a \emph{binary variable} $l_j^t$   for each leaf $j \in \Le.\tree$  where $t \in \M$ which will be either 0 or 1 depending on the leaf is active or not.

\smallskip
\noindent\textbf{Constraint set I: consistency of the tree encoding.}
We need to encode the semantics of the trees as predicates over the variables $\{p_i^k\}$ and $\{l_j^t\}$.
This part builds upon existing MILP encoding of tree ensembles from the literature~\citep{DBLP:conf/icml/KantchelianTJ16}, used in a different context, namely the verification of adversarial robustness.
We will briefly describe the basic idea, and the details are omitted.

\noindent
\emph{Correct assignment of the active leaves.}
This is encoded recursively. 
Consider the tree $\tree_i$, and suppose the root variable is labeled as $\{p_i^k\}$.
Suppose the set of leaves rooted in the ``True'' and ``False'' children of the root are $\{l_1^t,\ldots,l_m^t \}$ and $\{l_{m+1}^t,\ldots,l_{m+n}^t \}$, respectively.
Then clearly, $p_i^k=1$ iff one element of $\{l_1^t,\ldots,l_m^t \}$ is active, and all elements in $\{l_{m+1}^t,\ldots,l_{m+n}^t \}$ are false.
This is encoded as: $1-(l_{m+1}^t+\ldots +l_{m+n}^t) = p_i^k = l_1^t + \ldots + l_m^t$.
The condition for the internal nodes uses similar ideas but is slightly more involved, because it may happen that none of their subtrees contains a true leaf
$1-(l_{m+1}^t+\ldots +l_{m+n}^t) \ge p_i^k \ge l_1^t + \ldots + l_m^t$.

\noindent
\emph{Exactly one leaf is active per tree.}
Suppose $l_1^t,\ldots,l_n^t$ are all the leaves of the tree $\tree_i\in \M$.
This constraint is encoded as $l_1^t+\cdots + l_n^t=1$, guaranteeing that \emph{exactly} one leaf is active.

\noindent
\emph{Output consistency.}
Finally, we encode the tree $\tree_t$'s output as $\sum_{a=1}^{n} v_{tl} \cdot l^t_a $
We encapsulate all consistency constraints from above in the boolean predicate $\Phi_c(W)$ over all variables $W$, which is true iff the above constraints are satisfied.

\smallskip
\noindent\textbf{Constraint set II: glitches conditions.}
We now encode the conditions from Equations~\eqref{eq:glitch:ordering of inputs} and \eqref{eq:glitch:non-monotonicity}.
For this, we create three copies of $W$, and label them as $W^-$, $W$, and $W^+$.
These three copies contain the respective feature variable vectors  $x^-, x, x^+$, whose final valuations will represent the three points constituting a glitch.
Eqn.~\eqref{eq:glitch:ordering of inputs} will be encoded using the following boolean predicate:
\begin{align}\label{eq:MILP:Psi}
	\Psi(W^-,W,W^+)\coloneqq \bigvee_{i} \bigwedge_{\substack{j \neq i}} \left( \left( x^-_j = x_j=x^+_j \right) \land \left( x^-_i < x_i < x^+_i \right)   \right),
\end{align}
which is true iff there exists a feature dimension $i$ for which $x_i^-<x_i<x_i^+$ and for all $j\neq i$, $x_j^-=x_j=x_j^+$.
Next, Eqn.~\eqref{eq:glitch:non-monotonicity} is encoded as the boolean predicate:
\begin{multline*}
\Delta(W^-,W,W^+)\coloneqq
((c^- \geq 0 \land c^+ \geq 0) \lor (c^- < 0 \land c^+ < 0 ) )\\
\land ((c < 0 \land c^- \geq 0) \lor(c \geq 0\land c^+ < 0)),
\end{multline*}
which is true iff the ensemble outputs $\M(x^-)$, $\M(x)$, and $\M(x^+)$ follow a non-monotonic trend.

Finally, we need to encode Eqn.~\ref{eq:glitch:sharpness}, which becomes $\Gamma(c^-,c,c^+,x^-,x^+,\alpha)\coloneqq \min\{|c^--c|,|c-c^+|\} >= \alpha|x_i^+-x_i^-|$, and can be split into two constraints (i)~$|c^--c| >= \alpha|x_i^+-x_i^-|$ and (ii)~$|c-c^+| >= \alpha|x_i^+-x_i^-|$.
These constraints can be encoded using standard techniques in the MILP literature~\cite{williams2013model}, where we introduce additional continuous variables to represent the absolute values, and integer variables to keep track of the sign of the parts inside.
After that, a set of linear equations are introduced whose validity is equivalent to the validity of (i) and (ii).
We omit these details for brevity; as a matter of fact, most modern MILP solvers accept constraints in the format such as (i) and (ii).

\smallskip
\noindent\textbf{The final encoding.}
Finally, after assembling all the constraints, we obtain the following MILP instance for the problem $\TEGLITCHmax$:
\begin{align*}
          &\quad\max_{W^-,W,W^+,\alpha} \alpha\\
          &\text{subject to:}
         \quad\Phi_c(W^-)\land \Phi_c(W)\land \Phi_c(W^+) \\
              &\land \Psi(W^-,W,W^+)\land \Delta(W^-,W,W^+) \land \Gamma(c^-,c,c^+,x^-,x^+,\alpha).
\end{align*}
As a solution, we obtain valuations of all the variables from which a glitch of the maximum magnitude can be extracted.

\smallskip
\noindent\textbf{Solving the problems $\TEGLITCHi$, $\TEGLITCH$.}
The only difference is that $\alpha$ is now provided as a constant, and therefore the MILP instance becomes a constraint satisfaction problem, where the objective function can be replaced by a constant, say $1$.
Furthermore, for $\TEGLITCHi$, we are additionally provided a fixed $i$.
To accommodate this, we replace the constraint $\Psi$ from Eqn.~\ref{eq:MILP:Psi} with the following:
\begin{align*}
  \textstyle
 \Psi'(W^-,W,W^+,i)\coloneqq  \bigwedge_{\substack{j \neq i}} \left( x^-_j = x_j=x^+_j \right) \land \left( x^-_i < x_i < x^+_i \right).
\end{align*}
The solution of this constraint system is either SAT, in which case we can also extract a glitch with magnitude above $\alpha$, or UNSAT, in which case we can conclude that no such glitch was found.

\section{Experimental Evaluation}\label{sec:section5}
We evaluate the effectiveness and efficiency of our approach to find glitches in tree ensemble models.
We designed our experiments to address four research questions:
\noindent\textbf{RQ~1:} Do real world tree ensembles contain glitches, and how widespread are they?
\noindent\textbf{RQ~2:} What is the effect of adversarially robust training, monotonicity-constrained training on glitches?
\noindent\textbf{RQ~3:} Can we solve the problems $\TEGLITCHmax$, $\TEGLITCH$, and $\TEGLITCHi$ within reasonable time? How do our method scale with increasing number of trees in the ensemble? Are there glitches with large magnitudes?
\noindent\textbf{RQ~4:}	How do glitches perform compared to robustness violations and non-monotonicity as proxies for finding anomalies?

\noindent{\bf Benchmarks and experimental setup:} 
We implemented our algorithms in the tool $\ourtool~\footnote{This is second version of our tool ensense which was published in ICLR 2026~\cite{varshney2026data}: https://github.com/formal-trust-AI/ensense}$, which uses Gurobi~\cite{gurobi} as backend MILP solver.
We evaluate $\ourtool$ on publicly available benchmarks used in prior verification work~\cite{DBLP:conf/nips/ChenZS0BH19} and well-known tabular data sets~\cite{diabetes_34, adult_2, iranian_churn_563, statlog_german_credit_data_144, breast_cancer_wisconsin_diagnostic_17, spambase_94}.
Table~\ref{tab:benchmarks} contains details of all datasets and models; entries 1–12 are taken from~\cite{DBLP:conf/nips/ChenZS0BH19}, whereas 13–18 are the models that we trained ourselves with 6 different configurations using XGBoost~\cite{DBLP:conf/kdd/ChenG16}.
We ran experiments on two AMD EPYC 7742 64-Core processors (2.25 GHz, 256 hardware threads in total) and 660GB RAM with 16 thread limit on the MILP solver and time limit of 3600 seconds per instance.
\smallskip

\begin{table*}[t]
  \centering
  \small
  \resizebox{0.4\linewidth}{!}{%
\begin{tabular}{l@{}c@{}c@{}c@{}c@{\,}}
      \toprule
      \multicolumn{4}{c}{\textbf{Benchmarks}} \\
      \toprule
      SNO & ModelName & \#Trees & Dep. & \#Feat. \\
      \midrule
  1 & breastCancerR      & 4   & 4 & 10  \\
  2 & breastCancerU    & 4   & 5 & 10  \\
  3 & diabetesR            & 20  & 4 & 8  \\
  4 & diabetesU          & 20  & 5 & 8  \\
  5 & ijcnnR               & 60  & 8 & 22 \\
  6 & ijcnnU             & 60  & 8 & 22 \\
  7 & higgsR               & 300 & 8 & 28 \\
  8 & higgsU             & 300 & 8 & 20 \\
  9 & webspamR             & 100 & 8 & 254 \\
       \bottomrule
\end{tabular}
  }
  \resizebox{0.4\linewidth}{!}{%
\begin{tabular}{l@{}c@{}c@{}c@{}c@{\,}}
      \toprule
  \multicolumn{4}{c}{
  } \\
      \toprule
      SNO & ModelName & \#Trees & Dep. & \#Feat. \\
      \midrule
  10 & webspamU          & 100 & 8 & 254 \\
  11 & binaryMnistR       & 1000 & 6 & 784 \\
  12 & binaryMnistU    & 1000 & 4 & 784 \\
  13 & adult                       & \{200,500,800\} & \{4,5\} & 15 \\
  14 & churn                       & \{200,500,800\} & \{4,5\} & 21 \\
  15 & pimadiabetes                & \{200,500,800\} & \{4,5\} & 9  \\
  16 & germanCredit             & \{200,500,800\}  & \{4,5\} & 20 \\
  17 & breastCancer              & \{200,500,800\} & \{4,5\} & 30 \\
  18 &  spambase                   & \{200,500,800\} & \{4,5\} & 57 \\
       \bottomrule
    \end{tabular}
  }
    \caption{Benchmarks for our experiments. R and U stand for robust and unrobust versions
      of the models, respectively.}
    \label{tab:benchmarks}
\end{table*}

\begin{table}

      \centering
    \begin{minipage}[b]{0.5\textwidth}
    \scriptsize
    \setlength{\tabcolsep}{2.5pt}
    \renewcommand{\arraystretch}{1.08}
    \resizebox{1.4\linewidth}{!}{%
    \begin{tabular}{lrrrrrr}
    \toprule
    \textbf{Config} 
    & \textbf{Adult} 
    & \textbf{Churn} 
    & \textbf{Pima} 
    & \textbf{German} 
    & \textbf{Breast} 
    & \textbf{Spambase} \\
    & 
    & 
    & \textbf{Diabetes} 
    & \textbf{Credit} 
    & \textbf{Cancer} 
    & \\
    \midrule
    \textbf{t200\_d4} & 85.71 & 42.86 & 100.00 & 38.89 & 33.33 & 92.00 \\
    \textbf{t200\_d5} & 78.57 & 26.67 & 100.00 & 29.41 & 41.38 & 91.49 \\
    \textbf{t500\_d4} & 57.14 & 26.67 & 100.00 & 33.33 & 31.03 & 92.00 \\
    \textbf{t500\_d5} & 57.14 & 26.67 & 100.00 & 23.53 & 41.38 & 93.75 \\
    \textbf{t800\_d4} & 78.57 & 26.67 & 100.00 & 22.22 & 31.03 & 90.00 \\
    \textbf{t800\_d5} & 64.29 & 26.67 & 100.00 & 23.53 & 41.38 & 72.92 \\
    \bottomrule
    \end{tabular}
    }
    \end{minipage}
    \hfill
  \begin{minipage}[b]{0.3\textwidth}
    \resizebox{1\linewidth}{!}{%
        \begin{tabular}{lrr}
        \toprule
        \textbf{Model} & \textbf{Robust} & \textbf{Unrobust} \\
        \midrule
        Breast Cancer & 0.00  & 50.00  \\
        Diabetes      & 50.00 & 87.50  \\
        IJCNN         & 84.62 & 92.31  \\
        Webspam       & 97.50 & 100.00 \\
        Binary MNIST  & 12.50 & 57.50  \\
        Higgs         & 60.71 & 85.71  \\
        \bottomrule
        \end{tabular}
    }
  \end{minipage}
  \caption{Features exhibiting glitches under TE\_GLITCH($\alpha$,i) with $\alpha=1$.
 Percentage of features with glitches across six tabular datasets and six tree configurations (t = number of trees, d = maximum depth).}
  \label{tab:glitch_count}
\end{table}

\begin{figure}[t]
    \centering
    \includegraphics[width=0.5\linewidth,scale=0.3]{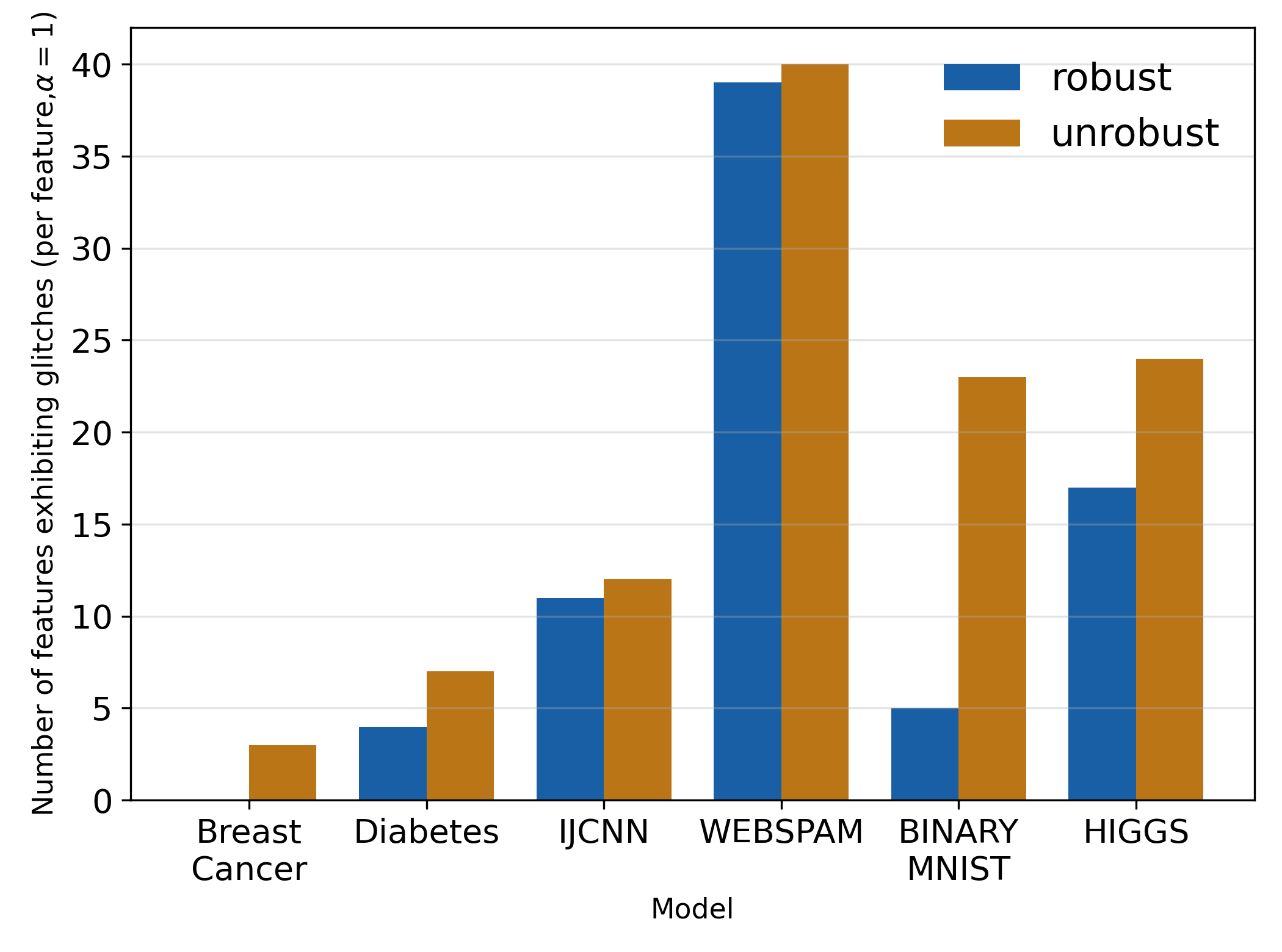}
    \caption{Robust vs.\ unrobust models. Robust training reduces but does not eliminate glitch-exhibiting features.}
    \label{fig:robglitch_count}
\end{figure}

\begin{figure}[t]
  \centering
  \begin{subfigure}[t]{0.5\textwidth}
    \centering
    \includegraphics[width=0.8\linewidth]{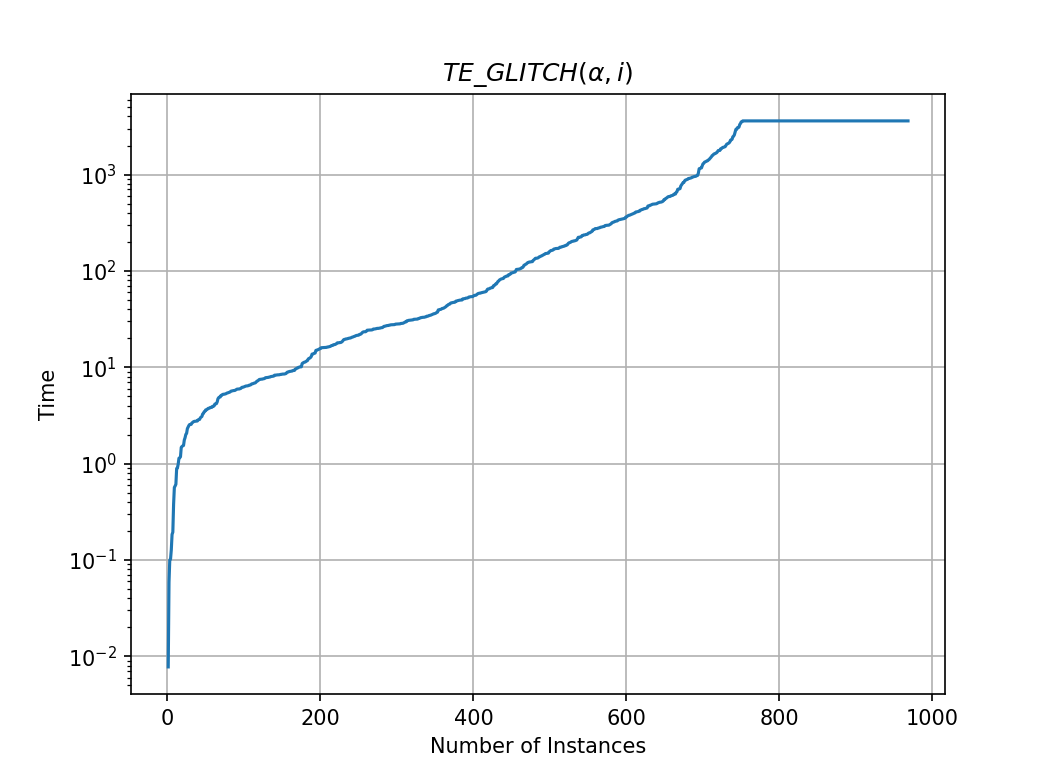}
    \caption{}\label{fig:prob3graph}
  \end{subfigure}
  \begin{subfigure}[t]{0.5\textwidth}
    \centering
    \includegraphics[width=0.8\linewidth]{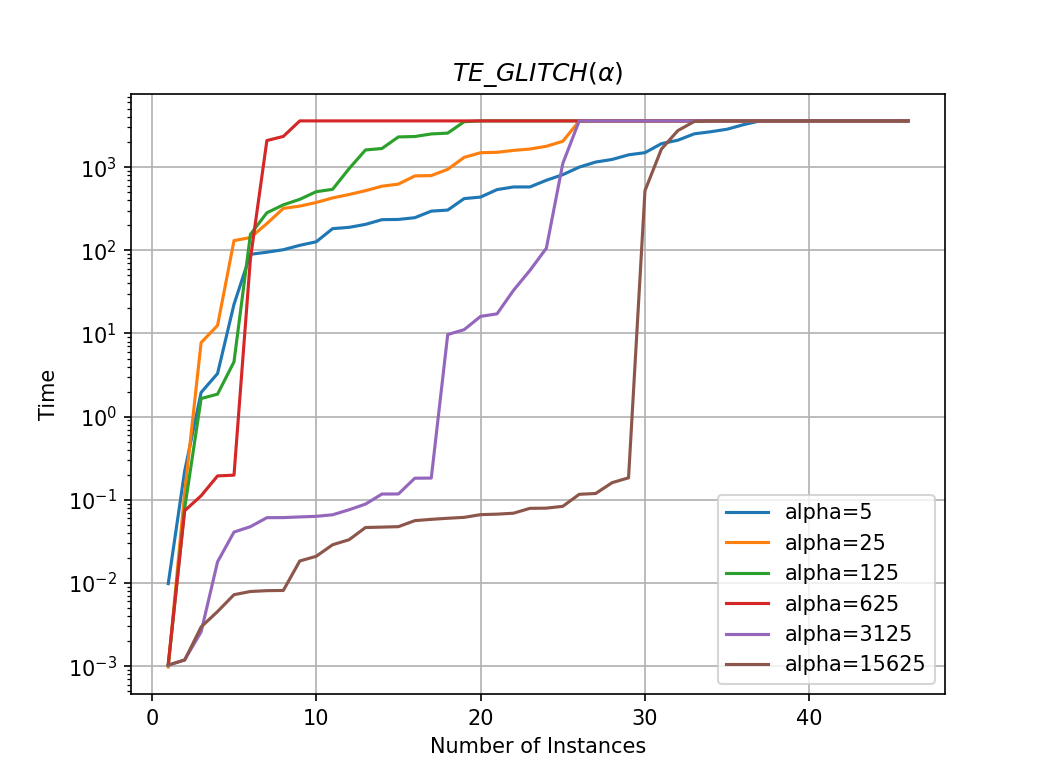}
    \caption{}\label{fig:prob2graph}
  \end{subfigure}
  \hspace{0.01\textwidth}
   \begin{subfigure}[t]{0.45\textwidth}
    \centering
    \includegraphics[width=0.8\linewidth]{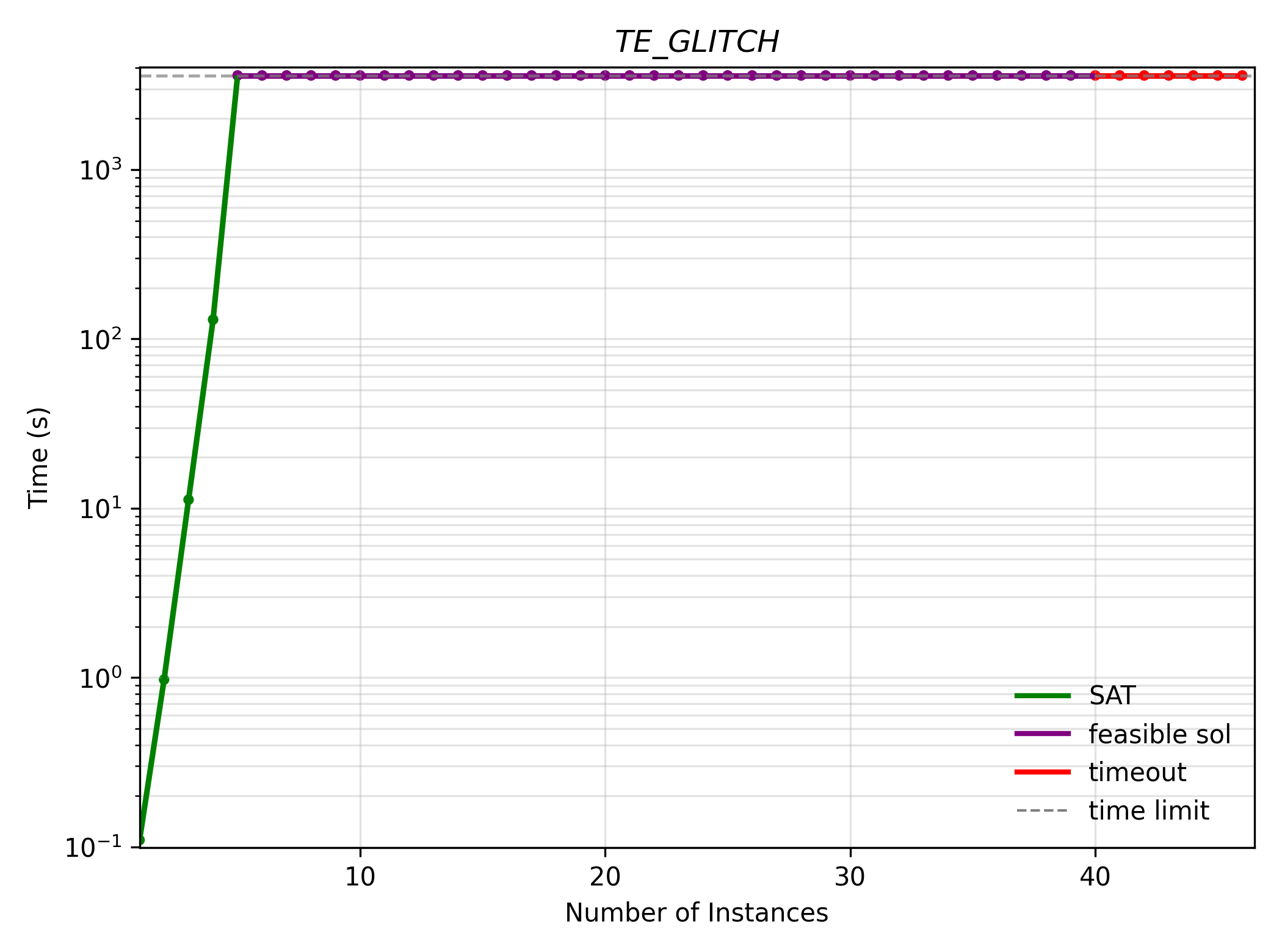}
   \caption{}\label{fig:prob1graph}
  \end{subfigure}
  \caption{Cactus plot reports the runtime (a) for $\TEGLITCHi$, (b) for  $\TEGLITCH$ and (c) $\TEGLITCHmax$.}\label{fig:all-probgraph}
\end{figure}

\begin{table*}[htbp]
  \centering
  \small
  \begin{minipage}[b]{0.6\textwidth}
  \setlength{\tabcolsep}{5pt}
  \renewcommand{\arraystretch}{1.05}
  \resizebox{1.1\linewidth}{!}{%
  \begin{tabular}{{@{}lcccccc@{}}}
      \toprule
      \textbf{Config} & \textbf{Adult} & \textbf{Churn} & \textbf{\shortstack{Pima\\Diabetes}} & 
      \textbf{\shortstack{German\\Credit}} & \textbf{\shortstack{Breast\\Cancer}} & \textbf{Spambase}  \\
      \midrule
      \textbf{t200\_d5} & 533.70 & 134.52 & 33.80  & 67.31  & 5.37  & 6399.41  \\
      \textbf{t200\_d6} & 474.30 & 151.22 & 68.24  & 26.36  & 17.85 & 2013.30  \\
      \textbf{t500\_d5} & 337.26 & 32.96  & 138.80 & 124.87 & 10.72 & TO       \\
      \textbf{t500\_d6} & 343.34 & 24.43  & 14.10  & 84.86  & 16.40 & 0.27     \\
      \textbf{t800\_d5} & 129.40 & 10.72  & 109.95 & 7.22   & 10.72 & TO       \\
      \textbf{t800\_d6} & 0.84   & TO     & 48.08  & 33.82  & 16.40 & TO       \\
      \bottomrule
  \end{tabular}
  }
  \end{minipage}
    \hfill
  \begin{minipage}[b]{0.3\textwidth}
\resizebox{1\linewidth}{!}{%
   \begin{tabular}{l@{}c@{}c}
      \toprule
      \textbf{Model} & \textbf{Robust} & \textbf{Unrobust} \\
      \midrule
      \textbf{Breast Cancer} & 0.40   & 14.33 \\
      \textbf{Diabetes}      & 8.16   & 133.70 \\
      \textbf{IJCNN}         & 333.54 & 43.11 \\
      \textbf{Higgs}         & TO     & TO \\
      \textbf{Webspam}       & TO     & TO \\
      \textbf{Binary MNIST}  & TO     & 7.12 \\
      \bottomrule
  \end{tabular}
}
    \end{minipage}
  \caption{Glitch magnitudes reported for $\TEGLITCHmax$.}
  \label{tab:robustUnrobustGlitchMagnitudes}
\end{table*}

\noindent
\textbf{Answering RQ1:} We answer this by solving the three problem variants of Sec.~\ref{sec:section2}.
For $\TEGLITCHi$ we fix $\alpha = 1$ which is the least permissive threshold in our evaluation.
 We run $\TEGLITCHi$ on every feature that have more than 2 splits except for models 9-12 where we only consider the top 40 most features that occur in the tree nodes. In total we obtain 1136 instances, of which 763 report a glitch. Table~\ref{tab:glitch_count}(a) summarizes these results and reports the percentage of  features that exhibit a glitch for the all the models mentioned in Table\ref{tab:benchmarks}. In every model, a substantial fraction of features exhibit a glitch.

For the remaining other two problems we only report aggregate counts. For $\TEGLITCH$, we have 48 instances per value of $\alpha$; the number reporting a glitch is 36, 26, 16, and 2 for $\alpha \in \{5, 25, 125, 625\}$ respectively, and none for $\alpha \in \{3125, 15625\}$. For $\TEGLITCHmax$, which searches for the maximum-magnitude glitch, 39 of 48 instances contain a glitch. Together, these results show that glitches are not isolated artifacts of a particular model or dataset, but a pervasive property of decision tree ensembles in practice.

\noindent
\textbf{Answering RQ2:} Table~\ref{tab:glitch_count}(right) compares the robust
and unrobust versions of six models: robust training lowers the number of
glitch-exhibiting features in every model, yet 5/6 still contain glitches afterward.
We presented the same data in Figure~\ref{fig:robglitch_count}.
For all datasets, we observe a reduction in the number of features exhibiting glitches.
We conclude that while robust training helps mitigate glitchy behavior, it does not eliminate it entirely.
Similarly we found the existence of glitches after monotonicity constraint
training which is shown in the extended version of this paper\cite{chandra2025glitchesdecisiontreeensemble}.

\noindent\textbf{Answering RQ3:}  In Figure~\ref{fig:all-probgraph}, we report the effectiveness of $\ourtool$ by solving the three problem variants introduced
in Sec. ~\ref{sec:section2}.
In figure~\ref{fig:prob3graph}, we report the cactus plot of the runtimes of
solving instances of problem $\TEGLITCHi$. We observe that $\ourtool$ can solve
986 out of 1136 instances within the time limit of 3600 seconds.
In figure~\ref{fig:prob2graph}, we report the cactus plot of the runtimes of
solving $\TEGLITCH$ instances. We observe that $\ourtool$ can solve 149 out of
288 instances within the time limit of 3600 seconds.
For $\TEGLITCHmax$, we consider all models. In total, we have 48 instances of $\TEGLITCHmax$. We report the cactus plot of the runtimes of solving these instances in figure~\ref{fig:prob1graph}.
For three instances, we find the maximal $\alpha$ for which a glitch exists. For
35 instances, our tool reaches timeout with a feasible solution, and for 9
instances, our tool reaches timeout with no solution.
Overall, we conclude that $\ourtool$ is able to verify glitches of
standard-sized models within reasonable time.
Finally, in Table~\ref{tab:robustUnrobustGlitchMagnitudes}, we report the
magnitudes of glitches that were found for various benchmarks, which
demonstrates that glitches with large
magnitude have widespread existence.

\begin{figure}[!  t]
    \centering
    \small

    \begin{minipage}[t]{0.37\linewidth}
    \vspace{0pt}
        \centering
        \small
        \resizebox{\linewidth}{!}{
        \begin{tabular}{lcccc}
            \hline
            Model & $\epsilon$ & \multicolumn{2}{c}{Robustness} & \#Glitch \\
            \cline{3-4}
                  &            & \#A & \#U &  \\
            \hline
            breastCancer & 0.3  & 3   & 3   & 0  \\
            diabetes     & 0.2  & 154 & 0   & 7  \\
            ijcnn        & 0.1  & 47  & 0   & 16 \\
            webspam      & 0.05 & 100 & 0   & 9  \\
            \hline
        \end{tabular}
        }

        \vspace{0.5em}
        \textbf{(a)}
    \end{minipage}
    \hfill
    \begin{minipage}[t]{0.28\linewidth}
    \vspace{0pt}
        \centering
        \pgfkeys{/pgf/number format/.cd,fixed}
        \begin{tikzpicture}[scale=0.75]
            \begin{axis}[
                scale only axis,
                width=\linewidth,
                height=0.8\linewidth,
                scaled ticks=false,
                xlabel={(b)},
                ylabel={},
                axis lines=left,
                enlarge x limits=true,
                enlarge y limits=true,
            ]
                \addplot table [x=f19, y=out, col sep=comma, mark=none]
                    {data/churn_t800_d5_mono_f18.csv};

                \addplot[mark=triangle*, yellow, draw=black, mark size=2pt]
                    coordinates {(0.79854,0.126982)};
                \node at (axis cs:0.80,0.10) {\footnotesize $p^-$};

                \addplot[mark=triangle*, yellow, draw=black, mark size=2pt]
                    coordinates {(0.80306,0.574298)};
                \node at (axis cs:0.804,0.6) {\footnotesize $p$};

                \addplot[mark=triangle*, yellow, draw=black, mark size=2pt]
                    coordinates {(0.839968,0.24195)};
                \node at (axis cs:0.95,0.1) {\footnotesize $p^+$};

                \draw[dashed] (axis cs:0.0,0.5) -- (axis cs:1,0.5);
            \end{axis}
        \end{tikzpicture}
    \end{minipage}
    \hfill
    \begin{minipage}[t]{0.28\linewidth}
    \vspace{0pt}
        \centering
        \pgfkeys{/pgf/number format/.cd,fixed}
        \begin{tikzpicture}[scale=0.75]
            \begin{axis}[
                scale only axis,
                width=\linewidth,
                height=0.8\linewidth,
                scaled ticks=false,
                xlabel={(c)},
                ylabel={},
                axis lines=left,
                enlarge x limits=true,
                enlarge y limits=true,
            ]
                \addplot table [x=f19, y=out, col sep=comma, mark=none]
                    {data/churn_t800_d5_mono_f18_19.csv};

                \addplot[mark=triangle*, yellow, draw=black, mark size=2pt]
                    coordinates {(0.061687500000000002,0.7711765541886264)};
                \node at (axis cs:0.001,0.75) {\footnotesize $p^-$};

                \addplot[mark=triangle*, yellow, draw=black, mark size=2pt]
                    coordinates {(0.11234150000000001,0.40088486341301777)};
                \node at (axis cs:0.05,0.4) {\footnotesize $p$};

                \addplot[mark=triangle*, yellow, draw=black, mark size=2pt]
                    coordinates {(0.1388975,0.6418342953888736)};
                \node at (axis cs:0.2388975,0.75) {\footnotesize $p^+$};

                \draw[dashed] (axis cs:0.0,0.5) -- (axis cs:1,0.5);
            \end{axis}
        \end{tikzpicture}
    \end{minipage}

    \caption{
    Local robustness violations and glitches. (a) Counts in $\epsilon$-neighborhoods of randomly sampled test points, where A and U denote anticipated and unanticipated robustness violations. (b)--(c) A churn-model glitch before and after monotonic training on $TotalCharges$.
    }
    \label{fig:monoplot}
\end{figure}

\noindent\textbf{Answering RQ4:}
Now we compare the concept of glitches with two other proxies for anomalies, namely, robustness violations and non-monotonicity.
To compare with robustness violations, we ask the following question in the experiments: if we randomly sample a data point and find a robustness violation around it, how likely is it that the violation is also a glitch? And how likely is it that the violation is anticipated (i.e., expected by the data)?
We took 4 robust models and their corresponding $\epsilon$ values from our benchmark. 
We randomly sample 500 data points from test set.
We locally searched local robustness violation using \textsc{treeVerify}~\cite{DBLP:conf/nips/ChenZS0BH19}.
We can easily solve such local search problem by modifying the range of values that $W^-,W,W^+$ are allowed to assume
Our choice of $\epsilon$ for local glitch search matches with~\cite{DBLP:conf/nips/ChenZS0BH19}.
We also searched the glitches using $\ourtool$ within the $\epsilon$ ball. 
We report our results in Figure~\ref{fig:monoplot}(a).
We observe that most of the robustness violations are 
anticipated, implying that checking robustness violations would overstate the number of anomalies.
On the other hand, glitches are rarer, because they simultaneously require a monotonicity violation.

\section{Discussion}\label{sec:section6}

We propose a formal model for \emph{glitches}, which enrich and extend the
existing trustworthiness metrics, like robustness and monotonicity.
We prove that verifying the existence of glitches in tree ensembles is an NP-complete problem, and we proposed MILP-based algorithms.
We demonstrate the practical usefulness of our tool $\ourtool{}$ on a range of benchmark examples collected from the literature.

Several future directions exist.
First, we have found evidences that glitches do exist in other forms of AI
models, including neural networks. In the Imagenet AlexNet
model~\cite{NIPS2012_c399862d}, we found the classification output of the model
oscillates within a narrow range of input perturbations, indicating glitches;
shown in the extended version of this paper\cite{chandra2025glitchesdecisiontreeensemble}.
Second, we only considered verification, but \textit{building} glitch-free
models remains an open question.
Finally, we will study root causes of glitches in various models.
We conjecture that most glitches originate due to lack of data availability in some parts of the input domain during training.

\subsubsection{Acknowledgements}
This research is supported by the State Bank of India (SBI) Foundation Hub for Data Science \& Analytics, IIT Bombay
and the project RYC2024-049116-I funded by MICIU/AEI/10.13039/501100011033 and the ESF+.

\bibliographystyle{splncs04}
  \bibliography{reference}

\end{document}